\documentclass[11pt]{article}
\usepackage[margin=1in]{geometry}
\usepackage{amsmath,amssymb,amsthm,booktabs,microtype,graphicx,mathtools}
\usepackage{titlesec}
\usepackage{placeins}
\usepackage[round,authoryear]{natbib}
\usepackage[colorlinks,citecolor=blue,linkcolor=blue]{hyperref}

\newtheorem{proposition}{Proposition}
\newtheorem{lemma}[proposition]{Lemma}
\newtheorem{example}[proposition]{Example}
\newtheorem{corollary}{Corollary}
\theoremstyle{definition}
\newtheorem{remark}{Remark}

\newcommand{\Aa}{\widetilde{A}}
\newcommand{\one}{\mathbf{1}}

\title{An equality condition for the Dobrushin bound on attention rollout and how
often it holds in trained transformers}
\author{Przemys\l{}aw Rola%
  \thanks{Department of Mathematics, Krak\'ow University of Economics, e-mail: \texttt{przemyslaw.rola@outlook.com}}}
\date{}

\begin{document}
\maketitle
\begingroup\renewcommand\thefootnote{}\footnotetext{
rollout, Dobrushin coefficient, stochastic matrices, residual connections,
transformers, attention sinks.}\endgroup

\begin{abstract}
The Dobrushin coefficient of each attention-rollout factor satisfies
$\kappa(\tfrac12(I+A))\le\tfrac12(1+\kappa(A))$, and multiplying these inequalities
over layers bounds the coefficient of the whole rollout. We characterise exactly
when the layerwise bound is tight: equality holds if and only if some token pair
attaining $\kappa(A)$ is \emph{mutually self-dominant} --- each of the two attends
to itself at least as strongly as the other attends to it. The condition is far
from automatic: uniformly random stochastic matrices satisfy it only $24$--$30\%$
of the time.

When tested on the head-averaged attention of each individual input and
restricted to content tokens --- image patches, words or tabular features,
excluding \texttt{cls}, register and separator tokens --- the condition holds for
essentially every input at every layer of DINOv2 (three model sizes), RoBERTa and
DistilBERT. In the supervised models DeiT-B and ViT-B/16 it holds for $91\%$ and
$64\%$ of input--layer pairs respectively, with all failures occurring late in
depth. In FT-Transformer trained on two standard tabular benchmarks it holds for
only $11$--$44\%$ of input--layer pairs. The special tokens account for almost all
failures in DINOv2 and the language models: when they are included, the
condition holds for only $82$--$97\%$ of input--layer pairs.
\end{abstract}

\section{Introduction}\label{sec:intro}

In a transformer, each attention layer mixes tokens. For every token it forms a
weighted average of all tokens' value vectors --- the weights being one row of a
row-stochastic attention matrix --- and adds the result to that token's
representation. Information is thereby redistributed across positions and the
representations are pulled toward one another. Iterating this aggregation can drive all tokens
toward a common representation ---
\emph{rank collapse} in the vision and language literature
\citep{dong2021}, \emph{oversmoothing} in the graph setting
\citep{oono2020,wu2023oversmoothing}. Whether, and how fast, a given trained
network approaches that regime is a question about a product of many mixing
operators, one per layer.

Attention rollout \citep{abnar2020rollout} is the standard way to summarise such
a product. It multiplies the attention matrices of successive layers into a single
row-stochastic operator describing how attention-mediated influence accumulates
with depth. To account for the residual (skip) connection, which passes each
token's representation to the next layer unchanged, each attention matrix $A$ is
first replaced by $\tfrac12(I+A)$, the identity $I$ standing for the skip
connection.
Because the result is row-stochastic, the classical theory of products of
inhomogeneous stochastic matrices applies to it directly.

\citet{biccari2026} make that connection explicit. They analyse rollout with
Doeblin--Dobrushin contraction theory, show that an operator with a small
Dobrushin coefficient is quantitatively close to a rank-one matrix, and control
the coefficient of the composed operator by a layerwise estimate: writing
$\Aa=\tfrac12(I+A)$ for a rollout factor,
$\kappa(\Aa)\le\tfrac12\big(1+\kappa(A)\big)$, which composes across $L$ layers
into a certificate carrying a factor $2^{-L}$. They also prove an unconditional
lower bound $\kappa(\Aa)\ge(N-2)/(2(N-1))$ on any single factor, and observe
empirically that the composed coefficient decreases with depth in trained
tabular transformers.

Their treatment therefore brackets the coefficient of each factor,
$\kappa(\Aa^{(\ell)})$, from above and from below, but leaves a question unanswered: \emph{when is the upper bound an
equality?} The question is not cosmetic. If the estimate is loose, the depth
certificate built from it may be improvable layer by layer, and the numbers it
produces say more about the estimate than about the network. If it is tight, no
layer's estimate can be improved; the depth certificate may still overstate the
coefficient of the whole rollout, but only because the coefficient of a product
of matrices can be smaller than the product of their coefficients.

This paper answers when the layerwise bound is an equality, and draws out the
consequences. Throughout, the object is the row-stochastic attention matrix $A$
--- the token-mixing pattern --- rather than the full multi-head attention
block of a layer, $X\mapsto\sum_hA_hXW_h^VW_h^O$ \citep{vaswani2017}, which also includes the
value and output projections,\footnote{In the notation of \citet{vaswani2017}, the multi-head block of a
self-attention layer, which all models here use, is
$\mathrm{Concat}(\mathrm{head}_1,\dots,\mathrm{head}_H)\,W^O$ with
$\mathrm{head}_h=\mathrm{Attention}(XW_h^Q,XW_h^K,XW_h^V)=A_hXW_h^V$, where $X$ is
the input of the block and $A_h$ is head $h$'s row-stochastic attention matrix.
Splitting $W^O$ into the row blocks $W_h^O$ belonging to the heads turns the
product into the sum $\sum_h\mathrm{head}_h\,W_h^O$. Biases add the same vector to
every token, since each $A_h$ has rows summing to one, and so do not affect the
token-mixing pattern.} so none of what follows certifies the complete computation;
(M1) in Section~\ref{sec:setup} states this formally.

\paragraph{Contributions.} (i) An exact characterisation of when the layerwise
bound of \citet{biccari2026} is an equality (Proposition~\ref{prop:exact}).
(ii) A test of the condition on the object their bound uses, the head-averaged
attention matrix of each input: on content tokens it holds for essentially every
input at every layer of DINOv2 and of two language models, less consistently in
two supervised vision transformers, whose failures come late in depth, and rarely
in FT-Transformer on tabular data; in DINOv2 and the language models the special
tokens --- \texttt{cls}, registers, separators --- account for almost all
failures. (iii) The observation that testing
attention averaged over inputs instead, a natural alternative, can both create
and destroy exactness. (iv) Two supporting measurements: the classical entry-wise
bound $\kappa(A)\le1-N\varepsilon$ is far looser than $\kappa(A)$ itself for
the numbers of tokens used in real transformers, and the Dobrushin coefficient of DINOv2's
attention averaged over heads and images is often attained at a pair involving a single
register token, and drops markedly when that token is left out.

Section~\ref{sec:setup} fixes notation and modelling choices,
Section~\ref{sec:exact} proves the characterisation, Section~\ref{sec:emp}
reports the measurements, Section~\ref{sec:limits} discusses limitations and
future work, and Section~\ref{sec:related} positions
the work.

\section{Setup}\label{sec:setup}

We adopt the setting and notation of \citet{biccari2026} throughout.

A transformer layer acts on a matrix $X\in\mathbb{R}^{N\times D}$ whose $N$ rows
are token representations of width $D$. Each layer adds its output to $X$ rather than
replacing it, $X\mapsto X+\mathrm{Layer}(X)$ --- a \emph{residual}, or skip,
connection --- so the representation passed from layer to layer, called the
\emph{residual stream}, accumulates the contributions of every layer. Self-attention with $H$ heads computes, for head $h$,
queries, keys and values $Q_h=XW_h^Q$, $K_h=XW_h^K$, $V_h=XW_h^V$ with
$W_h^Q,W_h^K,W_h^V\in\mathbb{R}^{D\times d_h}$, where $d_h$ is the per-head
query/key dimension ($d_h=D/H$ when the heads partition the width evenly, as in
the models studied here), and the attention matrix
\begin{equation}
A_h=\operatorname{softmax}\!\big(Q_hK_h^{\top}/\sqrt{d_h}\big)\in\mathbb{R}^{N\times N},
\end{equation}
row-wise, so that $A_h\one=\one$ and $(A_h)_{ik}\ge0$: each row is a
probability distribution over tokens. The head's contribution to the residual
stream is $A_hV_hW_h^O$, where $W_h^O$ is head $h$'s
slice of the output projection, $W_h^O\in\mathbb{R}^{d_h\times D}$; the layer sums these over heads, adds the
result to $X$ through the residual connection, and follows with normalisation and
an MLP. Throughout, $N$ denotes the number of tokens.

Let $A^{(\ell)}$ denote the (head-averaged, see (M2)) attention matrix of layer
$\ell=1,\dots,L$, and $\Aa^{(\ell)}=(1-\alpha)I+\alpha A^{(\ell)}$ the
corresponding residual factor. The depth-ordered product
\begin{equation}\label{eq:rollout}
R_L\;=\;\Aa^{(L)}\Aa^{(L-1)}\cdots\Aa^{(1)}
\end{equation}
is the transition matrix of a time-inhomogeneous Markov chain on the token index
(the layer applied first stands rightmost, since in this model representations
propagate as $X^{(\ell)}=\Aa^{(\ell)}X^{(\ell-1)}$); at $\alpha=\tfrac12$ it is attention
rollout \citep{abnar2020rollout}, and at $\alpha=1$ the pure-attention product.
\emph{Rank collapse} --- all tokens converging to a common representation --- is
then the \emph{weak ergodicity} of that chain: the rows of $R_L$ becoming
identical.

Dobrushin's ergodicity coefficient \citep{dobrushin1956,seneta2006} of a row-stochastic $A$ is
\begin{equation}\label{eq:dobrushin}
\kappa(A)\;=\;\tfrac12\max_{i,j}\big\|A_{i\cdot}-A_{j\cdot}\big\|_1
\;=\;1-\min_{i,j}\sum_{k}\min\!\big(A_{ik},A_{jk}\big)\;\in\;[0,1],
\end{equation}
with $\kappa(A)=0$ iff all rows coincide (i.e.\ $A$ has rank one).\footnote{We write $\kappa$ following \citet{biccari2026}; the Markov-chain
literature \citep{seneta2006,ipsen2011} writes $\tau_1$ for the same quantity.} Two properties make it
useful here: it contracts oscillation
$\operatorname{osc}(x)=\max_kx_k-\min_kx_k$, in the sense that
$\operatorname{osc}(Ax)\le\kappa(A)\operatorname{osc}(x)$,
and it is submultiplicative: for any row-stochastic
$M^{(1)},\dots,M^{(L)}$,
\begin{equation}\label{eq:submult}
\kappa\big(M^{(L)}\cdots M^{(1)}\big)\;\le\;\prod_{\ell=1}^{L}\kappa\big(M^{(\ell)}\big).
\end{equation}
We state it for general factors because it is applied below both to the
attention matrices $A^{(\ell)}$ and to the residual factors $\Aa^{(\ell)}$.
Both are classical \citep{seneta2006}. We also use that $\kappa$ is convex in
$A$, which is immediate from \eqref{eq:dobrushin}: $\kappa$ is a maximum of the
convex functions $A\mapsto\tfrac12\|A_{i\cdot}-A_{j\cdot}\|_1$.
\paragraph{Terminology.} We use a few terms in their technical senses.
For a row vector, $\|x\|_1=\sum_k|x_k|$; since attention rows are probability
distributions, $\tfrac12\|A_{i\cdot}-A_{j\cdot}\|_1$ is the total-variation
distance\footnote{In the usual convention,
$\mathrm{TV}(p,q)=\max_S|p(S)-q(S)|=\tfrac12\|p-q\|_1$; some authors call
$\|p-q\|_1$ itself the total variation.} between the distributions of tokens $i$
and $j$, so $\kappa(A)$ is the
total-variation distance between the two \emph{most dissimilar} attention rows.
A pair $(i,j)$ \emph{attains} $\kappa(A)$ when it realises that maximum; the
coefficient is decided entirely by such argmax pairs, and all other tokens are
invisible to it.
Concretely, the contraction property
$\operatorname{osc}(Ax)\le\kappa(A)\operatorname{osc}(x)$ says that, for any vector
$x$ with one value per token, the gap between the largest and smallest value
across tokens after mixing by $A$ is at most $\kappa(A)$ times what it was. At the level of token mixing, the residual
connection makes the operator acting on the residual stream the convex
combination $\Aa=(1-\alpha)I+\alpha A$ rather than $A$ itself, which is why every
statement below is about $\Aa$. By a \emph{certificate} we mean a bound
computable from forward passes alone, without training or search. An \emph{attention sink} \citep{xiao2024sinks} is a token receiving disproportionate attention
from many others. Finally, $\kappa$ relates to the more
familiar spectral notion by $\kappa(A)\ge|\lambda_2(A)|$ \citep{seneta1979}, where $\lambda_2(A)$ is the
eigenvalue of second-largest modulus (the largest is $1$), so
$1-\kappa(A)$ lower-bounds the spectral gap. The advantage of $\kappa$ over the
spectral gap here is that it is submultiplicative over \emph{different}
matrices, which is what a product of distinct layers requires.

Equation~\eqref{eq:submult} is a per-layer certificate obtained from a single forward
pass. Crucially it requires no time-homogeneity: it holds for arbitrary distinct
$A^{(\ell)}$, unlike results that assume the same parameters at every layer in order to
reach an asymptotic limit \citep{geshkovski2023,vo2025}.

\paragraph{Modelling choices.}
These are choices, not theorems, and every claim below is scoped to them.
\begin{enumerate}
\item[(M1)] A transformer layer is not a stochastic matrix. Only the token-mixing
part is; the output projection $W^O$, the MLP and the normalisation lie outside
the certificate.
\item[(M2)] With $H$ heads the layer is not a single chain: each head has its own
row-stochastic matrix $A_h$. Following \citet{biccari2026}, every statement
about a layer in Sections~\ref{sec:exact} and~\ref{sec:emp} refers to the head
average $\bar A=\frac1H\sum_h A_h$, which is again row-stochastic, and by
convexity $\kappa(\bar A)\le\frac1H\sum_h\kappa(A_h)$. Exactness, however, is
not preserved by averaging: it can hold for $\bar A$ while failing for
individual heads. We therefore also test the individual heads
(Section~\ref{sec:perinput}).
\item[(M3)] The residual branch is modelled as $\Aa=(1-\alpha)I+\alpha A$. At
$\alpha=\tfrac12$ this is exactly a rollout factor
\citep{abnar2020rollout,biccari2026}; we keep general $\alpha$
because Proposition~\ref{prop:exact} holds for every $\alpha\in(0,1)$, not only
for $\alpha=\tfrac12$. DINOv2 uses
LayerScale (a learned per-channel multiplier on each branch), so the effective
residual weight is not literally $\tfrac12$.
\item[(M4)] We test the \emph{content subchain}: the submatrix of $\bar A$ on the
content tokens alone --- image patches, words, or tabular features --- discarding
the special tokens, namely the \texttt{cls} token of the image models and of
FT-Transformer, the four registers of the DINOv2 models that have them, and the
classification and separator tokens \texttt{[CLS]} and \texttt{[SEP]} of the
language models (\texttt{<s>} and \texttt{</s>} in RoBERTa), with the remaining rows rescaled to sum to
one. For images we also call it the patch subchain.
This is a choice about \emph{which tokens} enter each attention matrix, separate
from how many inputs the matrix is averaged over. The subchain is a different
Markov chain from the original, not a restriction of it --- deleting states and
renormalising changes the dynamics --- so results on it do not transfer
automatically to the full token set. We report every model on both, and the
difference is itself informative (Section~\ref{sec:emp}). Because rescaling the
rows does not commute with averaging, we apply the restriction only after all
averaging: over heads and, where applicable (Section~\ref{sec:averaged}), over
inputs as well.
\end{enumerate}

\section{The residual bound and when it is attained}
\label{sec:exact}

\begin{lemma}[Residual bound; {\citealp[eq.~(A.6)]{biccari2026}} at $\alpha=\tfrac12$]
\label{prop:bound}
For row-stochastic $A$ and $\alpha\in[0,1]$, with $\Aa=(1-\alpha)I+\alpha A$,
\begin{equation}
\kappa(\Aa)\;\le\;(1-\alpha)+\alpha\,\kappa(A).
\end{equation}
\end{lemma}
\begin{proof}
$\kappa$ is convex on the set of row-stochastic matrices and $\kappa(I)=1$, so
$\kappa((1-\alpha)I+\alpha A)\le(1-\alpha)\kappa(I)+\alpha\kappa(A)$. At
$\alpha=\tfrac12$ this recovers the estimate of \citet{biccari2026}, obtained
there by the triangle inequality.
\end{proof}

At $\alpha=\tfrac12$, \citet{biccari2026} compose their case of
Lemma~\ref{prop:bound} across depth into
\begin{equation}\label{eq:their-product}
\kappa(R_L)\;\le\;\prod_{\ell=1}^{L}\kappa\big(\Aa^{(\ell)}\big)
\;\le\;2^{-L}\prod_{\ell=1}^{L}\big(1+\kappa(A^{(\ell)})\big),
\end{equation}
and separately prove the unconditional floor
$\kappa(\Aa)\ge(N-2)/(2(N-1))$,\footnote{In their notation, $d$ stands for
$N$.} which tends
to $\tfrac12$ as the number of tokens grows. We state the results below for every $\alpha\in(0,1)$,
which the proof allows at no extra cost; all our measurements use
$\alpha=\tfrac12$.

The interesting question is when the second inequality --- the layerwise
estimate --- is tight, since that determines whether the certificate is
informative or merely valid.

\begin{proposition}[Exactness]\label{prop:exact}
Let $A$ be row-stochastic and $0<\alpha<1$. Then
\begin{equation}
\kappa(\Aa)=(1-\alpha)+\alpha\,\kappa(A)
\end{equation}
if and only if there exists a pair $(i,j)$ attaining
$\tfrac12\|A_{i\cdot}-A_{j\cdot}\|_1=\kappa(A)$ that satisfies
\begin{equation}\label{eq:cond}
A_{ii}\ \ge\ A_{ji}
\qquad\text{and}\qquad
A_{jj}\ \ge\ A_{ij}.
\end{equation}
\end{proposition}
\begin{proof}
Fix $i\neq j$ and write $d=A_{i\cdot}-A_{j\cdot}$, so $d_i=A_{ii}-A_{ji}$ and
$d_j=A_{ij}-A_{jj}$. Then
\begin{equation*}
\Aa_{i\cdot}-\Aa_{j\cdot}=(1-\alpha)(e_i-e_j)+\alpha d,
\end{equation*}
with $e_i$ the $i$-th canonical basis vector, whose coordinates are
$(1-\alpha)+\alpha d_i$ at position $i$, $-(1-\alpha)+\alpha d_j$ at position
$j$, and $\alpha d_k$ elsewhere. By the triangle inequality,
\begin{align*}
\tfrac12\bigl\|\Aa_{i\cdot}-\Aa_{j\cdot}\bigr\|_1
&=\tfrac12\Bigl(\bigl|(1-\alpha)+\alpha d_i\bigr|+\bigl|(1-\alpha)-\alpha d_j\bigr|
+\alpha\textstyle\sum_{k\neq i,j}|d_k|\Bigr)\\
&\le(1-\alpha)+\alpha\cdot\tfrac12\|d\|_1\;\le\;(1-\alpha)+\alpha\,\kappa(A).
\end{align*}
The first inequality holds term by term: since $1-\alpha>0$,
$\bigl|(1-\alpha)+\alpha d_i\bigr|\le(1-\alpha)+\alpha|d_i|$ with equality if and
only if $d_i\ge0$, and $\bigl|(1-\alpha)-\alpha d_j\bigr|\le(1-\alpha)+\alpha|d_j|$
with equality if and only if $d_j\le0$ (as $|x+y|=|x|+|y|$ if and only if
$xy\ge0$). So the first inequality is an equality if and only if $(i,j)$
satisfies~\eqref{eq:cond}, and the second if and only if $(i,j)$ attains
$\kappa(A)$.

If some $\kappa(A)$-attaining pair satisfies~\eqref{eq:cond},
equality holds throughout for it; since the chain bounds every pair,
$\kappa(\Aa)=(1-\alpha)+\alpha\kappa(A)$.

Conversely, if
$\kappa(\Aa)=(1-\alpha)+\alpha\kappa(A)$, equality holds throughout for any pair
attaining $\kappa(\Aa)$, which has $i\neq j$ because this value is positive; as
$\alpha>0$, that pair attains $\kappa(A)$ and satisfies~\eqref{eq:cond}.
\end{proof}

\begin{example}[The bound can fail completely]\label{ex:swap}
Let $A=\begin{psmallmatrix}0&1\\1&0\end{psmallmatrix}$, the two-token swap. Its
rows are disjoint, so $\kappa(A)=1$ and Lemma~\ref{prop:bound} gives only
the trivial bound $\kappa(\Aa)\le1$. But
$\Aa=\tfrac12(I+A)=\begin{psmallmatrix}1/2&1/2\\1/2&1/2\end{psmallmatrix}$ has
identical rows, hence $\kappa(\Aa)=0$: the residual destroys the entire
difference instead of preserving half of it. The condition fails at the only
attaining pair, since $A_{11}=0<1=A_{21}$.
\end{example}

\begin{remark}[Interpretation]
Condition~\eqref{eq:cond} says that each of the two extremal tokens attends to
itself at least as strongly as the other extremal token attends to it --- a
\emph{mutual self-dominance} of a pair that realises the coefficient (the pair
need not be unique; the condition asks this of some attaining pair, not of all).
The residual scales every row by $\alpha$ and adds mass $1-\alpha$ to its
diagonal entry; mutual self-dominance is precisely the condition under which,
for that pair, the added mass reinforces the existing difference between the two
rows rather than partially cancelling it. Example~\ref{ex:swap} is the opposite
extreme, where it cancels the difference entirely. The condition fails at a pair when one of its two
tokens receives more attention from the other than it pays itself --- a pattern
that an attention sink, receiving attention from many tokens, can produce when it
is one of the pair.
\end{remark}

\begin{remark}[Size of a failure]\label{rem:size}
The computation in the proof also bounds the size of a failure. For any pair
$(i,j)$ attaining $\kappa(A)$,
\[
0\le(1-\alpha)+\alpha\kappa(A)-\kappa(\Aa)\le
\alpha\bigl[(A_{ji}-A_{ii})^+ + (A_{ij}-A_{jj})^+\bigr],
\]
and the bound on the gap is sharp: in Example~\ref{ex:swap} both sides equal $1$. A failure is therefore never larger than $\alpha$ times the total amount
by which the two inequalities of~\eqref{eq:cond} fail, and each of these amounts is
at most one attention entry. When attention is spread over many tokens its
entries are small, and so are its failures: uniformly random matrices with $256$
tokens (the uniform model below) fail condition~\eqref{eq:cond} about three times
in four, but at $\alpha=\tfrac12$ by a median gap of only $0.5\%$ of $\tfrac12\kappa(A)$.
\end{remark}

\begin{corollary}[Factorwise tightness]\label{cor:floor}
Let $0<\alpha<1$. If the condition of Proposition~\ref{prop:exact} holds at
every layer, then $\kappa(\Aa^{(\ell)})=(1-\alpha)+\alpha\kappa(A^{(\ell)})$ for
all $\ell$. In particular, at $\alpha=\tfrac12$ the second inequality
in~\eqref{eq:their-product} is an equality:
\begin{equation*}
\prod_{\ell=1}^{L}\kappa\big(\Aa^{(\ell)}\big)\;=\;2^{-L}\prod_{\ell=1}^{L}\big(1+\kappa(A^{(\ell)})\big).
\end{equation*}
\end{corollary}

The depth certificate~\eqref{eq:their-product} bounds $\kappa(R_L)$ in two steps:
submultiplicativity, $\kappa(R_L)\le\prod_\ell\kappa(\Aa^{(\ell)})$, and then the
layerwise estimate applied to each factor. Under the condition the second step
is an equality at every layer, so the certificate cannot be improved by
sharpening the layerwise estimate. Whatever gap remains between the certificate
and $\kappa(R_L)$ comes from the first step, which is typically strict: the
coefficient of a product of matrices is usually much smaller than the product of
their coefficients. A tighter certificate would have to use more information
about the attention matrices than their coefficients.

Condition~\eqref{eq:cond} cannot be dropped: without it the bound can fail, and
Example~\ref{ex:swap} shows that it can fail by the largest possible amount. How
often random matrices satisfy the condition depends on how they are generated, so
we use three random models, each at $N=4$, $8$, $32$, $128$ and $256$ tokens with
$1000$ matrices per size. Under the \emph{uniform} model each row is a random
probability vector, with all probability vectors equally likely: it is drawn
independently from the uniform distribution on the probability simplex
$\{p\in\mathbb{R}^N:p_k\ge0,\ \sum_kp_k=1\}$ (the Dirichlet distribution with all
parameters equal to one); under the \emph{softmax} model each row is
the softmax of independent standard Gaussian logits. Both satisfy the condition
in only $24$--$30\%$ of cases, at every size. The third model imitates two
features common in trained attention: every token sends an extra share of its
attention to one common \emph{sink} token, and an extra share to itself. Row $i$
is $(1-a-b)\,u_i+a\,e_s+b\,e_i$, where $u_i$ is a random row drawn from the
uniform model, $e_s$ and $e_i$ are the unit vectors of the sink token $s$, chosen uniformly at
random, and of token $i$, and $a,b\in[0,\tfrac12]$ are the weights of the two extra
shares. Since $a+b\le1$, each row is again a probability vector. The weights are drawn once per
matrix, since a sink attracts attention from every token of a head; to include
matrices with neither feature, either one, or both, $a$ and $b$ are drawn
independently, each equal to zero with probability $\tfrac12$ and otherwise
uniform on $[0,\tfrac12]$. The extra weight on the diagonal favours the
condition, which asks each extremal token to attend to itself at least as
strongly as the other attends to it, and the condition now holds in $58$--$63\%$
of cases. The exact rate depends on these choices; the point is that building in
self-preference makes the condition more common, but still far from certain. It is
therefore not automatic, which is what makes its behaviour in trained models
(Section~\ref{sec:emp}) a substantive observation rather than a restatement of
Lemma~\ref{prop:bound}.

Two remarks delimit what the depth bound can say. It is one-sided: a small
product of layer coefficients shows that the rollout is close to rank one, but a
large product shows nothing, since $\kappa(R_L)$ can lie far below it. And where
the condition fails the per-layer estimate is loose, so the bound understates how
much such a layer mixes: a failure of the condition is a limitation of the bound,
not a property of the network.

\section{Empirical results}\label{sec:emp}

\paragraph{What is tested.} In the bound of \citet{biccari2026}, the layerwise
estimate is applied to one input at a time. For each input $x$, the heads of
layer $\ell$ are averaged into the input-dependent matrix $\bar A^{(\ell)}(x)$,
and the coefficient of that input's rollout is bounded by~\eqref{eq:their-product}.
For a dataset, the input-specific rollouts are averaged, and the coefficient of
the average is bounded by the average of the right-hand sides
of~\eqref{eq:their-product}, as in their (A.7) and (A.8). Each factor
$\tfrac12\big(1+\kappa(\bar A^{(\ell)}(x))\big)$ in these bounds is the layerwise
estimate of Lemma~\ref{prop:bound} applied to the matrix of one input at one
layer. Whether their estimate is tight is therefore a question about
the matrices $\bar A^{(\ell)}(x)$ of individual inputs, and this is our primary
test: for every input and layer we evaluate the gap of
Proposition~\ref{prop:exact} on that input's head-averaged matrix, and report the
share of input--layer pairs at which the bound is attained, that is, at which
condition~\eqref{eq:cond} holds. Averaging over the inputs before testing gives a
different object, $\langle\bar A^{(\ell)}\rangle$, the mean of $\bar A^{(\ell)}(x)$ over the inputs, which
behaves differently; we report it separately (Section~\ref{sec:averaged}).

\paragraph{Models and data.} We test DINOv2 \citep{oquab2024dinov2} in its
released variant with register tokens, at three sizes (ViT-S/14, ViT-B/14 and
ViT-L/14, with $12$, $12$ and $24$ layers), and, as a control, the ViT-B/14
variant trained without them. \emph{Register tokens} are extra learnable vectors appended to the input sequence that
correspond to no image patch: \citet{darcet2024registers} introduced them because
vision transformers otherwise repurpose a few patch tokens to carry high-norm
global information. We further test two ViT-B/16 models trained with supervision
on ImageNet, the original ViT \citep{dosovitskiy2021vit} and DeiT-B
\citep{touvron2021deit}, and the bidirectional language encoders RoBERTa-base
\citep{liu2019roberta} and DistilBERT \citep{sanh2019distilbert}. Image models
are evaluated on $512$ images drawn at random from ImageNetV2
\citep{recht2019imagenetv2}, language models on $512$ sequences of exactly $128$
tokens with no padding, drawn at random from the training split of WikiText-103
\citep{merity2017}, so that every attention row is a distribution over the
full sequence. For every model we test attention among the content
tokens, image patches, words or tabular features. The special tokens --- the \texttt{cls} token,
the registers, and the classification and separator tokens \texttt{[CLS]} and \texttt{[SEP]}
(\texttt{<s>} and \texttt{</s>} in RoBERTa) --- remain part of the model and
influence its attention, but they are removed from the matrices we test, which is
the content subchain of (M4). We also test every model on all of its tokens; the
difference shows where the condition breaks. We chose
bidirectional encoders rather than decoder-only models such as GPT-2 for a reason
specific to the coefficient. In a decoder-only model each token may attend only to
itself and earlier tokens, so the first token attends only to itself: its
attention row is $e_1=(1,0,\dots,0)$. That row overlaps any other row $j$ only in
$A_{j1}$, the attention token $j$ pays to the first position, so by the second
form of~\eqref{eq:dobrushin}, $\kappa(A)\ge1-\min_jA_{j1}$. The next rows are
constrained in the same way, more weakly: row $i$ can overlap another row only on
the first $i$ positions. The coefficient is therefore governed by how much the
tokens attend to the first few positions, a consequence of the mask, rather than
by mixing across the sequence.

\paragraph{Sample size and precision.} The shares in Table~\ref{tab:perinput} use
$512$ inputs per pretrained model. Larger samples give the same picture: on all $10\,000$
images of ImageNetV2 for DINOv2-B, DeiT-B and ViT-B/16, and on $4096$ sequences
for RoBERTa, the overall share changes by at most $0.2$ percentage points and the
share at any single layer by at most $2.6$; in particular, DINOv2-B fails at no
layer beyond the first for any of the $10\,000$ images.\footnote{The larger
samples also reveal failures too rare for $512$ inputs to detect at two further
layers: in $4$ of $10\,000$ images at layer $6$ of ViT-B/16, and in $2$ at layer
$10$ of DeiT-B.} We count a gap as zero when it is below
$10^{-10}$; Appendix~\ref{app:precision} explains this choice and shows that the
counts barely depend on it.

\subsection{The condition on each input's attention}\label{sec:perinput}

\begin{table}[t]
\centering
\small
\setlength{\tabcolsep}{3.5pt}
\begin{tabular}{llcccc}
\toprule
& & \multicolumn{3}{c}{per input} & averaged\\
\cmidrule(lr){3-5}\cmidrule(lr){6-6}
model & training & content tokens & all tokens & heads & exact layers\\
\midrule
DINOv2-S               & self-supervised & $6\,144/6\,144$ ($100\%$)     & $91.2\%$ & $97.2\%$ & $12/12$\\
DINOv2-B               & self-supervised & $6\,137/6\,144$ ($99.9\%$)    & $83.8\%$ & $95.9\%$ & $12/12$\\
DINOv2-L               & self-supervised & $12\,285/12\,288$ ($99.98\%$) & $81.9\%$ & $91.4\%$ & $24/24$\\
DINOv2-B, no registers & self-supervised & $6\,141/6\,144$ ($99.95\%$)   & $97.4\%$ & $94.0\%$ & $12/12$\\
DeiT-B                 & supervised      & $5\,582/6\,144$ ($90.9\%$)    & $90.7\%$ & $82.3\%$ & $11/12$\\
ViT-B/16               & supervised      & $3\,954/6\,144$ ($64.4\%$)    & $64.1\%$ & $60.7\%$ & $11/12$\\
RoBERTa-base           & masked LM       & $6\,125/6\,144$ ($99.7\%$)    & $89.7\%$ & $97.3\%$ & $12/12$\\
DistilBERT             & distillation    & $3\,072/3\,072$ ($100\%$)     & $96.8\%$ & $94.2\%$ & $6/6$\\
\midrule
FT-T, California, 3 blocks & supervised  & $18.7\%$ & $16.1\%$ & --- & $2/6$\\
FT-T, California, 6 blocks & supervised  & $11.3\%$ & $9.8\%$  & --- & $3/15$\\
FT-T, Higgs, 3 blocks      & supervised  & $44.4\%$ & $32.3\%$ & --- & $3/6$\\
\bottomrule
\end{tabular}
\caption{How often condition~\eqref{eq:cond} holds: per input, the object of the
depth bound of \citet{biccari2026}, and for attention averaged over inputs.
\emph{Content tokens}: share of input--layer pairs for which the input's
head-averaged attention matrix, restricted to content tokens (M4), satisfies the
condition, across $512$ images or sequences per pretrained model. \emph{All
tokens}: the same share without restricting to content tokens. \emph{Heads}:
share of head--input--layer triples --- one head, at one layer, on one input ---
for which that head's attention matrix, restricted to content tokens, satisfies
the condition, across the same $512$ inputs; not
computed for FT-T. \emph{Averaged}: number of layers for which
$\langle\bar A^{(\ell)}\rangle$, restricted to content tokens, satisfies the
condition (Section~\ref{sec:averaged}). The content tokens are the $256$ patches
of DINOv2, the $196$ patches of DeiT-B and ViT-B/16, the $126$ words of the
language models, and the $8$ or $28$ features of FT-Transformer (FT-T,
Section~\ref{sec:ftt}), which is evaluated over three seeds with $4000$ test rows
each.}
\label{tab:perinput}
\end{table}

In DINOv2 the condition holds for essentially every image at every layer
(Table~\ref{tab:perinput}): at $24\,566$ of the $24\,576$ image--layer pairs of
DINOv2-S, -B and -L, the variants trained with registers, against the $24$--$30\%$ at which random
matrices satisfy it (Section~\ref{sec:exact}). All ten failures are at layer~$1$ --- unlike those of the other
encoders below, which come late in depth --- and all are small. We measure the
size of a failure by its gap as a fraction of $\tfrac12\kappa(\bar A)$, the part
of the bound that depends on the attention: under exactness, at our
$\alpha=\tfrac12$, $\kappa(\Aa)=\tfrac12+\tfrac12\kappa(\bar A)$, and the constant
$\tfrac12$ comes from the skip connection and is the same for every matrix. By this measure,
at each of the ten failures the bound overstates the attention's part by at most
$3\%$, and their gaps, between
$3.6\times10^{-5}$ and $3.6\times10^{-3}$, are far above numerical precision
(Appendix~\ref{app:precision}). DINOv2-B trained without registers behaves the
same way, so training with registers is not what makes the patch subchain exact.
The two language models behave like DINOv2 on their content tokens: RoBERTa
satisfies the condition at $6\,125$ of its $6\,144$ sequence--layer pairs and
DistilBERT at all $3\,072$.

The two supervised vision transformers are the exception. Even on their patch
tokens the condition fails for some inputs, late in depth
(Figure~\ref{fig:perinput}): in DeiT-B at layer $11$ for $10\%$ of images and at
layer $12$ for all of them, and in ViT-B/16 for most images at layers $8$, $9$,
$10$ and $12$, and for $27\%$ and $46\%$ of them at layers $7$ and $11$. Most of
these failures are small in the sense just defined: at the median failure the
bound overstates the attention's part by $1.6\%$ in DeiT-B and $3.7\%$ in ViT-B/16.
With so many tokens this is expected rather than informative
(Remark~\ref{rem:size}): what distinguishes trained attention is how often the
gap is exactly zero. We do
not know why they arise late in depth, or why only in these two models.

\paragraph{Special tokens are where the condition breaks.} Including the special
tokens changes the picture for DINOv2 and the language models
(Table~\ref{tab:perinput}, bottom row of Figure~\ref{fig:perinput}). On all
tokens, DINOv2 with registers satisfies the condition for only $82$--$91\%$ of
image--layer pairs: at layer $1$ nearly every image fails, and in the later
layers of DINOv2-B and -L many do. RoBERTa and DistilBERT drop to $90\%$ and
$97\%$, failing mainly at layers $9$ and $12$ and at layer $6$ respectively; removing
the two special tokens of each sequence, which leaves only its $126$ text tokens,
removes $713$ of the $732$ failures. This is consistent
with the observation of \citet{clark2019} that over half of BERT's attention in layers
$6$ to $10$ goes to \texttt{[SEP]}. Without registers, the \texttt{cls} token
alone lowers DINOv2-B to $97\%$, mostly at layer $11$. In the two supervised
vision transformers, by contrast, including the \texttt{cls} token changes almost
nothing ($90.7\%$ and $64.1\%$ on all tokens, against $90.9\%$ and $64.4\%$ on
patches): their failures lie among the patches themselves.

\IfFileExists{fig_perinput.pdf}{%
\begin{figure}[t]\centering
\includegraphics[width=\textwidth]{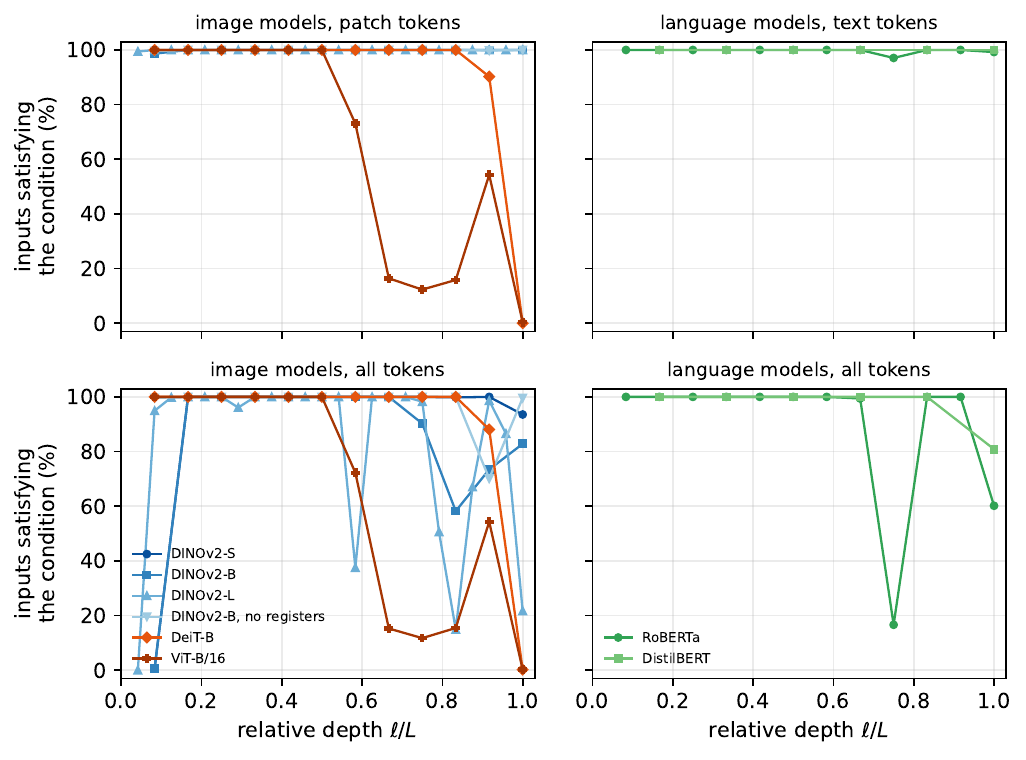}
\caption{Percentage of the $512$ images or sequences whose head-averaged
attention matrix at that layer satisfies condition~\eqref{eq:cond}, layer by layer, against relative depth $\ell/L$, for the pretrained
models of Table~\ref{tab:perinput} ($512$ images or sequences each), on content
tokens (top) and on all tokens (bottom). In the top-left panel the four DINOv2
curves lie on top of one another at $100\%$. Appendix~\ref{app:perlayer} gives
the exact number of failing inputs at every layer where there are any.}
\label{fig:perinput}\end{figure}}{}

Individual heads, reported in the last-but-one column of
Table~\ref{tab:perinput}, satisfy the condition for $61$--$97\%$ of
head--input--layer triples (one head, at one layer, on one input), far more often than random matrices, so the head average of (M2) is not
what produces the result.

\paragraph{What the certificate attains.} For a single image $x$, the depth
bound~\eqref{eq:their-product} of \citet{biccari2026} is
\[
B(x)\;=\;\prod_{\ell=1}^{L}\tfrac12\Big(1+\kappa\big(\bar A^{(\ell)}(x)\big)\Big),
\]
an upper bound on $\kappa(R_L(x))$, the coefficient of that image's rollout. For
a dataset they bound the coefficient of the averaged rollout
$\langle R_L\rangle$, the mean of $R_L(x)$ over the images, by the mean of
$B(x)$.

Where condition~\eqref{eq:cond} holds, the corresponding factor of $B(x)$ is exact; in
DINOv2 this is so for all but ten of the $24\,576$ factors. Yet the bound remains
loose. Over our $512$ images the mean of $B(x)$ is $0.44$, $0.36$ and $0.065$
for DINOv2-S, -B and -L, while the coefficient it bounds,
$\kappa(\langle R_L\rangle)$, is $0.093$, $0.075$ and $0.0022$ on the same
images: the bound is five to thirty times too large, although almost every factor
in it is exact.

The excess comes from the two other steps in its derivation: submultiplicativity,
$\kappa(R_L(x))\le\prod_\ell\kappa(\Aa^{(\ell)}(x))$, since the coefficient of a
product of matrices is usually much smaller than the product of their
coefficients; and averaging over images, which by convexity can only lower the
coefficient.

\paragraph{The attention analysis predicts a collapse that does not happen.}
If tokens were mixed by attention alone, as the rollout models them, DINOv2-L
would leave every token with nearly the same representation: its rollout over $24$
layers, averaged over images, has coefficient $0.0022$, so by Theorem~2.3 of
\citet{biccari2026} it is close to a matrix whose rows are all equal --- one that
gives every token the same weighted average of the inputs. Yet DINOv2-L does not collapse. Its patch
features support per-pixel tasks such as semantic segmentation and depth
estimation from a single image \citep{oquab2024dinov2}, which would be impossible
if the patch tokens had become interchangeable. The resolution lies in what the
rollout leaves out. It tracks only the attention matrices, so the value
projections, the MLP and the normalisation lie outside it (M1), and it weights
the residual branch by $\tfrac12$, whereas DINOv2 weights it by learned
per-channel factors (M3). Whatever keeps the representations apart must lie in
these parts of the network. This is a caution about the rollout framework itself --- ours and
\citet{biccari2026}'s alike --- rather than a finding about DINOv2: a
near-collapsed rollout is compatible with a model that is working perfectly well.

\subsection{Averaging over inputs}\label{sec:averaged}

Averaging attention over inputs before testing is natural when one wants to
describe a layer's typical attention pattern. The
last column of Table~\ref{tab:perinput} reports it: for each model $\langle\bar A^{(\ell)}\rangle$ is
formed from the same inputs and tested. On content tokens, the averaged matrices of the eight pretrained models satisfy
the condition at $100$ of their $102$ layers; the two exceptions are the last
layers of DeiT-B and ViT-B/16. This verdict is stable: on the larger samples of
Section~\ref{sec:emp} --- all $10\,000$ images for DINOv2-B, DeiT-B and ViT-B/16, and
$4096$ sequences for RoBERTa --- it is the same at every layer, and for all three
DINOv2 models also over five independent draws of $512$ images.

But the averaged test answers a different question, and it can agree or
disagree with the per-input test, because exactness is not preserved by
averaging in either direction. In ViT-B/16 the averaged matrix is exact at
layers $7$ to $11$, although $27\%$ to $88\%$ of individual images fail there:
averaging \emph{creates} exactness. In RoBERTa on all tokens, $509$ of $512$
sequences are exact at layer $8$, but the averaged matrix is not (its gap is
$7.4\times10^{-3}$): averaging \emph{destroys} it. Averaging also lowers the coefficient. For DINOv2-B, $\langle\bar A^{(\ell)}\rangle$
has $\kappa$ between $0.19$ and $0.78$ across layers, while the median over
images of the per-image coefficient $\kappa(\bar A^{(\ell)}(x))$ lies between $0.43$
and $0.94$, because an average of attention
patterns is more uniform than the patterns themselves. The coefficient of an
averaged matrix describes how much the typical pattern mixes, not how much the
attention of any input does.

Individual heads averaged over inputs, on content tokens, are exact for $570$ of
the $600$ DINOv2 heads, $141$ of $144$ in RoBERTa and $70$ of $72$ in
DistilBERT. On all tokens the averaged matrices fail more often, as the
per-input test does: at five of the twelve layers of DINOv2-B, for instance, each
time with a register in the extremal pair.

\paragraph{Column dominance is sufficient but not the explanation.} One simple
sufficient condition exists. If every diagonal entry is the largest in its column
--- $A_{ii}\ge A_{ki}$ for all $i$ and $k$, so that no token attends to $i$ more
strongly than $i$ attends to itself --- then both inequalities of~\eqref{eq:cond}
hold at every pair, and exactness follows whichever pair attains $\kappa$. The
property is demanding: among the uniform and softmax random matrices of
Section~\ref{sec:exact}, none of the $1000$ at each of $8$, $32$, $128$ and $256$
tokens has it. It is common among the DINOv2
heads averaged over images, $\langle A_h^{(\ell)}\rangle$: $308$ of the $600$
have it, and all of them satisfy condition~\eqref{eq:cond}, as the argument
requires. But it does not explain the regularity: $262$ of the $570$ heads that
satisfy the condition lack it, and $\langle\bar A^{(\ell)}\rangle$ satisfies the
condition at all $48$ layers of the three models, although it is column dominant
at only $29$ of them, and at four of them not in a single column. Which property governs the condition
remains open.

\subsection{A contrasting case: FT-Transformer}\label{sec:ftt}

Condition~\eqref{eq:cond} is not universal. We trained FT-Transformer \citep{gorishniy2021},
a widely used transformer architecture for tabular data, on two standard
benchmarks, California Housing \citep{pace1997} with $8$ features and Higgs
\citep{baldi2014} with $28$, using its authors' own data splits, preprocessing
and default configuration of three blocks, and additionally with six blocks on
California Housing; its tokens are one \texttt{cls} token and one token per
feature, $9$ in all on California Housing and $29$ on Higgs (details in
Appendix~\ref{app:tabular}). The models reproduce the
authors' own results: on Higgs our three seeds reach $0.725$--$0.728$ accuracy
against their $0.721$--$0.724$, and on California Housing they come within about
$0.01$ of their RMSE. Yet per input, on its feature tokens, condition~\eqref{eq:cond} holds for only $11$--$44\%$
of input--layer pairs, and for $10$--$32\%$ with the \texttt{cls} token included
(Table~\ref{tab:perinput}): far below the other models, and in two of the three
configurations, both on California Housing, below the rate for random
matrices. Few tokens do not by themselves explain this: random matrices satisfy the
condition about as often with few tokens as with many
(Section~\ref{sec:exact}). The size of the failures
depends on the dataset: at the median failure the bound overstates the
attention's part of the coefficient, $\tfrac12\kappa(\bar A)$, by $5\%$ on Higgs and by $23$--$38\%$ on California Housing.
These larger failures are consistent with Remark~\ref{rem:size}: with only $29$ or
$9$ tokens, single attention entries are large.

\subsection{The entry-wise route does not bind at realistic numbers of tokens}
\label{sec:entrywise}

Some asymptotic arguments that attention mixes tokens rest on a lower bound,
uniform over all layers, for the entries of the attention matrices. \citet{wu2023oversmoothing}, for instance,
use a constant $\varepsilon>0$
bounding every entry of every attention matrix from below, which they show
exists (their Lemma~2). Such an
$\varepsilon$ bounds the coefficient directly: if every entry is at least
$\varepsilon$, any two rows overlap in at least $N\varepsilon$, so by the second
form of \eqref{eq:dobrushin}, $\kappa(A)=1-\min_{i,j}\sum_k\min(A_{ik},A_{jk})
\le1-N\varepsilon$ --- the classical Doeblin bound. The question is how large
$\varepsilon$ is in a real transformer.

It is small. We measure it on the patch subchain of DINOv2, in the case most
favourable to the entry-wise route: the matrix $\langle\bar A^{(\ell)}\rangle$, whose smallest entry is
at least the smallest entry over all individual heads and inputs, because
averaging fills in near-zero entries. Even there the Doeblin bound is loose rather
than informative: it exceeds the directly computed coefficient of that matrix,
$\kappa(\langle\bar A^{(\ell)}\rangle)$, at
every layer (Figure~\ref{fig:entrywise}), by a factor between $1.06$ and $4.5$,
and on average by $24$ to $41\%$ per layer in the three models, an excess that
compounds over depth. The bound that the asymptotic arguments need must hold for every
head and every input at once; its $\varepsilon$ can therefore be no larger than
the smallest entry in this favourable case, and the bound it gives can only be
looser. This is not a defect in the asymptotic results, which concern limits and
are correct as such; it is a statement about what the same quantity can certify
for the number of tokens of a deployed model.

\IfFileExists{fig_entrywise.pdf}{%
\begin{figure}[t]\centering
\includegraphics[width=\textwidth]{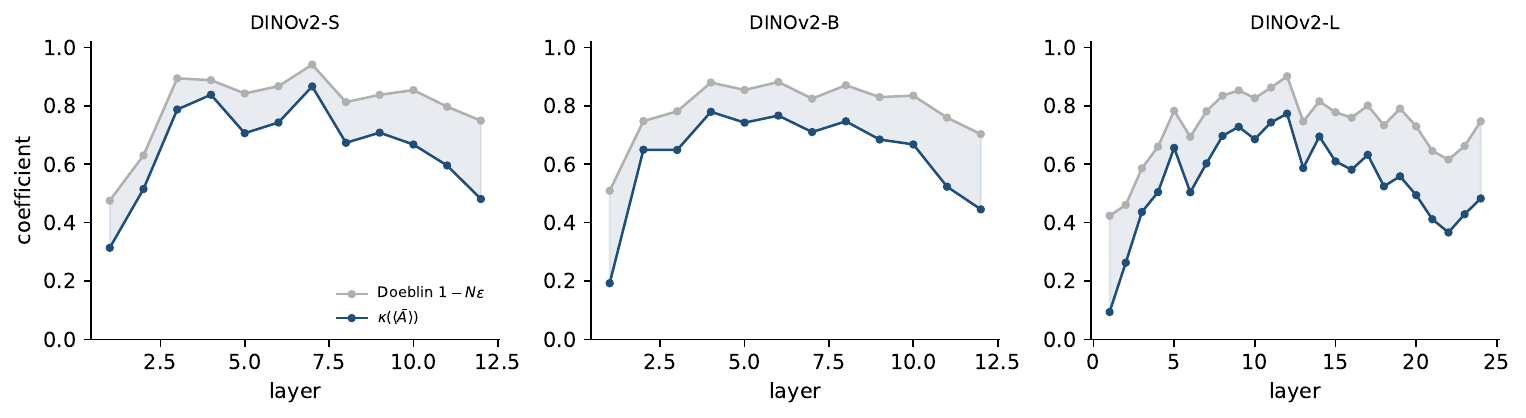}
\caption{Doeblin bound $1-N\varepsilon$ against the directly computed
coefficient, per layer, for $\langle\bar A^{(\ell)}\rangle$ on the patch subchain of DINOv2-S,
-B and -L. The shaded gap is what the entry-wise route gives up.}
\label{fig:entrywise}\end{figure}}{}

\subsection{When one token carries the coefficient}\label{sec:carrier}

Because $\kappa(A)=\max_{i}c_i$ with $c_i:=\max_j\tfrac12\|A_{i\cdot}-A_{j\cdot}\|_1$,
the coefficient decomposes over tokens. A sharper question is which tokens are
\emph{load-bearing}: with
$\kappa_{-i}(A):=\tfrac12\max_{j,k\ne i}\|A_{j\cdot}-A_{k\cdot}\|_1$ --- the
coefficient restricted to pairs avoiding token $i$, rows unchanged, \emph{not}
the deleted-and-renormalised submatrix --- define
$\Delta_i:=\kappa(A)-\kappa_{-i}(A)$, the drop when token $i$ is barred from the
maximum. By construction $\Delta_i>0$ only for tokens belonging to every
$\kappa$-attaining pair, so at most two tokens have a positive score.

We apply it to $\langle\bar A^{(\ell)}\rangle$ on the full DINOv2 token set, registers included, and call the token with the largest $\Delta_i$ the layer's
\emph{carrier}. The top carriers fall into two well-separated groups (Table~\ref{tab:carrier},
Figure~\ref{fig:carrier}). At roughly half the layers of each model the carrier
is a register or \texttt{cls} token with $\Delta$ up to $0.16$; at the others it is
an ordinary patch with $\Delta$ of at most a few thousandths, meaning that no
single token matters. In the sample of Table~\ref{tab:carrier} and
Figure~\ref{fig:carrier} the two groups overlap only at layer $23$ of DINOv2-L,
where a register carries $\Delta=1.8\times10^{-3}$; in four further independent
draws of $512$ images, \texttt{cls} carries that layer and the groups do not
overlap. The carrier is also usually one \emph{specific} register
(Table~\ref{tab:carrier}, last column), and this is stable across the five
draws: the modal register is the same in every draw for all three models, and
the carrier token itself is identical across draws at all $12$ layers of
DINOv2-S and at $23$ of the $24$ layers of DINOv2-L. The other registers and
\texttt{cls} carry the rest. The four register tokens
\citep{darcet2024registers} are therefore not interchangeable with respect to
mixing.

\IfFileExists{fig_carrier.pdf}{%
\begin{figure}[t]\centering
\includegraphics[width=\textwidth]{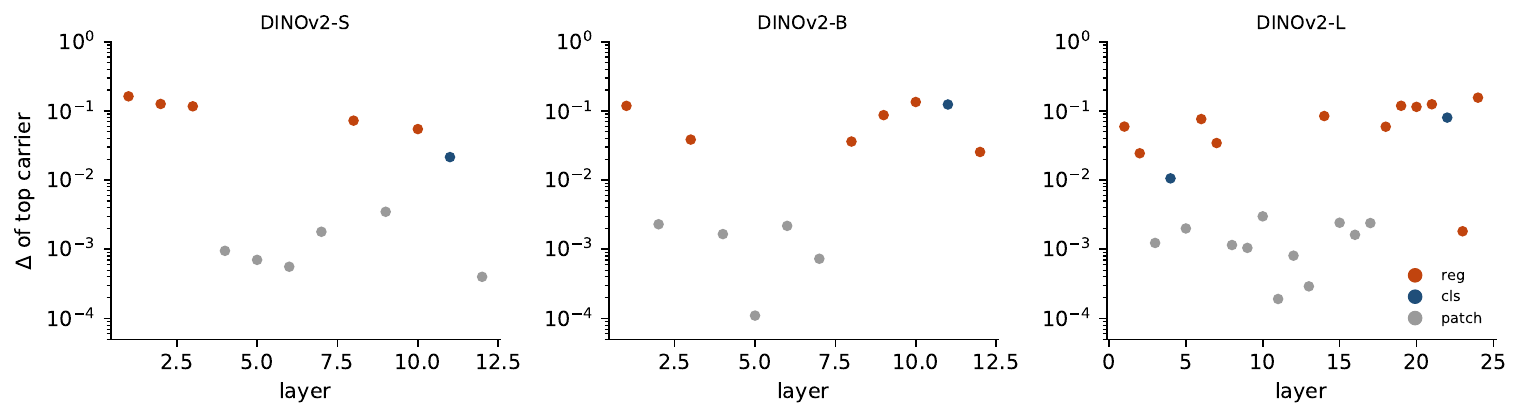}
\caption{Leave-one-out score of the top carrier at each layer, by token role
(log scale), for $\langle\bar A^{(\ell)}\rangle$ of DINOv2, all tokens included. Register and \texttt{cls} carriers mostly sit one to two orders of
magnitude above patch carriers; the exception is layer $23$ of DINOv2-L.}
\label{fig:carrier}\end{figure}}{}

\begin{table}[t]
\centering
\small
\begin{tabular}{lccccc}
\toprule
model & $L$ & layers with reg/\texttt{cls} carrier & $\max_\ell\Delta$
& median $\Delta$ (patch carrier) & modal carrier\\
\midrule
DINOv2-S & 12 & 6 & $0.163$ & $8.3\times10^{-4}$ & \texttt{reg\#1} ($4/6$)\\
DINOv2-B & 12 & 7 & $0.135$ & $1.7\times10^{-3}$ & \texttt{reg\#4} ($5/7$)\\
DINOv2-L & 24 & 13 & $0.156$ & $1.2\times10^{-3}$ & \texttt{reg\#1} ($8/13$)\\
\bottomrule
\end{tabular}
\caption{Leave-one-out token scores for $\langle\bar A^{(\ell)}\rangle$ of DINOv2, all tokens
included. For each token $i$, $\Delta_i$ is the drop in
$\kappa$ when $i$ is excluded from the maximum, and the layer's carrier is the
token with the largest $\Delta_i$. Columns: the number of layers whose carrier is
a register or \texttt{cls}; the largest $\Delta$ over all layers; the median
$\Delta$ at layers whose carrier is a patch; and the modal carrier, the token that
is the carrier most often, with the number of register-or-\texttt{cls} layers it
carries.}
\label{tab:carrier}
\end{table}

\section{Limitations and future work}\label{sec:limits}

The exactness result is an empirical regularity, not a property of trained
transformers in general. On content tokens it holds for essentially every input
in DINOv2 and the two language models we test; the two supervised vision
transformers fail it for some inputs late in depth, including the special
tokens breaks it in DINOv2 and the language models, and FT-Transformer trained on tabular data largely
lacks it. We do not know what governs it --- why the special tokens break it, or
why the supervised models fail late --- and column dominance, the obvious
candidate, is sufficient but does not explain it (Section~\ref{sec:averaged}). The evidence has further limits. Each pretrained
model is tested on one dataset, and the tabular evidence rests on one architecture and
two datasets. The residual weight $\alpha=\tfrac12$ follows
\citet{biccari2026} (M3), and the
main results concern content tokens (M4). We have not tested
decoder-only models, where causal masking changes what the coefficient measures
(Section~\ref{sec:emp}).

Three directions follow. The first is to explain the pattern: to find the
property of trained attention that the condition tracks, why the special tokens
break it, and why it fails late in depth in the supervised vision transformers
but not in DINOv2. A related question is whether the columns in which column
dominance fails can locate such tokens --- attention sinks, or the high-norm
artefact patches of models trained without registers
\citep{darcet2024registers} --- from attention alone. The second is to widen the
evidence, to further tabular architectures such as TabPFN and to larger and
decoder-only language models. The third is practical: the coefficient can be
computed from the attention matrices of a single forward pass, and we plan to
examine whether it can support more efficient training or the pruning of
attention heads.

\section{Related work}\label{sec:related}

\paragraph{Direct antecedent.}
\citet{biccari2026} apply Doeblin--Dobrushin contraction theory to attention
rollout in tabular transformers, and the framing adopted here is theirs: the
Dobrushin coefficient of the head-averaged attention matrix, the residual
modelled as $\tfrac12(I+A)$, the layerwise estimate
$\kappa(\tfrac12(I+A))\le\tfrac12(1+\kappa(A))$ (their (A.6), our
Lemma~\ref{prop:bound}) and its composition over depth, the entry-wise bound
$\kappa\le1-N\varepsilon$ (their Proposition~A.2), and the observation that the
coefficient decreases with depth. Their principal result, Theorem~2.3, shows that
a rollout with a small coefficient is close to a rank-one matrix. We take up a
question their treatment leaves open: they bound the coefficient of each residual
factor from above, $\kappa(\Aa^{(\ell)})\le\tfrac12\big(1+\kappa(A^{(\ell)})\big)$
(their (A.6)), but do not determine when this bound is attained.
Proposition~\ref{prop:exact} answers that question, and
Corollary~\ref{cor:floor} turns the answer into a statement about their depth
bound.

\paragraph{Rank collapse and oversmoothing.}
That pure self-attention loses rank with depth is established:
\citet{dong2021} show doubly exponential convergence to a rank-one matrix, and
identify residual connections and the MLP as what prevents it. The residual
connection leaves the same mark on the Dobrushin coefficient of each rollout
factor: at $\alpha=\tfrac12$, $\kappa(\Aa^{(\ell)})\ge(N-2)/(2(N-1))\approx\tfrac12$
at every layer \citep[Lemma~A.4]{biccari2026}, so a bound built layer by layer
cannot certify contraction faster than about a factor of two per layer. And
Section~\ref{sec:emp} shows that what prevents collapse lies outside what the
rollout models:
in DINOv2-L the attention rollout is close to rank one at $24$ layers, yet the
representations do not collapse. In the graph setting the same
phenomenon is known as oversmoothing \citep{oono2020}.

\paragraph{Ergodicity of attention products.}
\citet{wu2023oversmoothing} analyse attention-based graph neural networks as
products of inhomogeneous row-stochastic matrices and prove oversmoothing at an
exponential rate; transformers fall within their setting as attention on a
complete graph. Their analysis concerns the limit of infinite depth; ours concerns the residual
factor of each layer at finite depth.

\paragraph{Depth as a dynamical system.} \citet{jacobs2026} treat ViT depth as a flow and
segment it into contiguous phases of reused computation. \citet{vo2025} characterise token limit points in deep selective SSMs
under an assumption of time-independent parameters, and \citet{geshkovski2023} study
clustering in self-attention dynamics in a continuous-time limit. The coefficient used
here requires neither homogeneity nor a limit, at the cost of giving a bound rather than
a rate.

\paragraph{Ergodicity coefficients.} $\kappa$ is due to \citet{dobrushin1956}; weak
ergodicity of inhomogeneous chains to \citet{hajnal1958}; \citet{seneta1979,seneta2006}
develops the general theory, and \citet{ipsen2011} treat the family of
coefficients defined by vector norms.

\appendix
\titleformat{\section}[block]{\normalfont\Large\bfseries}{Appendix \thesection}{1em}{}

\section{FT-Transformer on tabular data}\label{app:tabular}

We train FT-Transformer with the authors' package \citep{gorishniy2021} on their
own data files, which fix the split of each dataset into training, validation and
test sets. Preprocessing follows their code and is the same for both datasets:
missing values, where present, are replaced by training-set column means; a small
noise of relative size $10^{-3}$ is added to the training features before a
quantile transform to a normal distribution is fitted, which breaks ties among
repeated values; and the targets of California Housing, the regression task, are
standardised. We use their default configuration of three blocks and, on
California Housing, also six blocks. In the last block the
package computes attention only for the \texttt{cls} token, so that block's
attention is a single row rather than a token-mixing matrix, and we test the
other blocks. Training stops once validation performance has not improved for
$16$ epochs, with no hyperparameter tuning. Attention is tested on $4000$ test
rows, each row being one input: per row, and averaged over rows. We run three seeds per configuration and
compare with the authors' own runs for the same seeds, and with gradient boosting
(scikit-learn's \texttt{HistGradientBoosting}, default settings) and a linear
model (linear regression for California Housing, logistic regression for Higgs)
fitted on the same split (Table~\ref{tab:ftt}).

\begin{table}[h]
\centering
\small
\setlength{\tabcolsep}{3pt}
\begin{tabular}{lccccccccc}
\toprule
configuration & $N$ & metric & FT-T & authors & boosting & linear & features & all & averaged\\
\midrule
California, 3 blocks & $9$  & RMSE     & $0.473$--$0.475$ & $0.463$--$0.465$ & $0.463$ & $0.731$ & $18.7\%$ & $16.1\%$ & $2/6$\\
California, 6 blocks & $9$  & RMSE     & $0.468$--$0.475$ & ---              & $0.463$ & $0.731$ & $11.3\%$ & $9.8\%$  & $3/15$\\
Higgs, 3 blocks      & $29$ & accuracy & $0.725$--$0.728$ & $0.721$--$0.724$ & $0.716$ & $0.630$ & $44.4\%$ & $32.3\%$ & $3/6$\\
\bottomrule
\end{tabular}
\caption{FT-Transformer (FT-T), three seeds per configuration. Performance
entries give the lowest and highest test value over the seeds: our runs (FT-T),
the authors' own runs with the same seeds (available for their default of three
blocks), gradient boosting and a linear model. $N$ is the number of tokens,
including \texttt{cls}. \emph{Features} and \emph{all}: share of row--layer pairs
for which the row's head-averaged attention matrix satisfies
condition~\eqref{eq:cond}, restricted to the feature tokens and on all tokens,
over the $4000$ test rows and three seeds. \emph{Averaged}: number of tested
layers, over the three seeds, for which $\langle\bar A^{(\ell)}\rangle$,
restricted to the feature tokens, satisfies it.}
\label{tab:ftt}
\end{table}

\section{Numerical precision}\label{app:precision}

The models compute attention in single precision, which carries about seven
significant digits. Computing the coefficient and the gap in single precision as
well would blur exactly the distinction we need: on random matrices of the same
size that satisfy the condition, the gap computed in single precision comes out as
large as $1.2\times10^{-7}$, the size of the smallest genuine failures we
observe, while in double precision it stays below $3\times10^{-16}$. We therefore
convert each attention matrix to double precision, with about sixteen
significant digits, before computing its coefficient and gap. The conversion is
exact, and a gap that is zero in exact arithmetic then appears as a number of
order $10^{-16}$, far below any genuine failure.

We count a gap as zero when it is below $10^{-10}$. For the head-averaged
matrices of the pretrained models this choice hardly matters: of the
$562\,752$ such gaps we computed, on content tokens and on all tokens, all but
one are either below $10^{-14}$ or at least $10^{-7}$, so any threshold between those values gives the same counts ---
a separation that exists only because the gaps are computed in double precision
--- and the one exception, $6\times10^{-8}$, alters a single count. For individual heads
a few gaps fall in between: moving the threshold anywhere in that range changes
the shares by at most $0.25$ percentage points on content tokens, and by up to
$1.6$ on the full token sets of DINOv2. Finally, as a cross-check, evaluating
condition~\eqref{eq:cond} directly, by locating the pairs that attain $\kappa$,
gives the same verdict as the gap at all $6144$ image--layer pairs of DINOv2-S and
of ViT-B/16 --- including the $2190$ pairs of ViT-B/16 at which the condition
fails, so the two tests agree where the bound is not attained as well as where it
is.

\section{Failures layer by layer}\label{app:perlayer}

Table~\ref{tab:perlayer} lists, for each pretrained model of
Table~\ref{tab:perinput}, the layers at which some input fails condition~\eqref{eq:cond} and
the number of failing inputs out of $512$, on content tokens and on all tokens.
At every other layer every input satisfies it.

\begin{table}[h]
\centering
\small
\begin{tabular}{lcp{9.2cm}}
\toprule
model & layers & failing inputs out of $512$, by layer\\
\midrule
\multicolumn{3}{l}{\emph{content tokens}}\\
DINOv2-S               & $12$ & none\\
DINOv2-B               & $12$ & layer $1$: $7$\\
DINOv2-L               & $24$ & layer $1$: $3$\\
DINOv2-B, no registers & $12$ & layer $1$: $2$; layer $11$: $1$\\
DeiT-B                 & $12$ & layer $11$: $50$; layer $12$: $512$\\
ViT-B/16               & $12$ & layer $7$: $138$; $8$: $428$; $9$: $449$; $10$: $431$; $11$: $234$; $12$: $510$\\
RoBERTa-base           & $12$ & layer $9$: $15$; layer $12$: $4$\\
DistilBERT             & $6$  & none\\
\midrule
\multicolumn{3}{l}{\emph{all tokens}}\\
DINOv2-S               & $12$ & layer $1$: $509$; $10$: $1$; $12$: $33$\\
DINOv2-B               & $12$ & layer $1$: $509$; $9$: $50$; $10$: $214$; $11$: $136$; $12$: $88$\\
DINOv2-L               & $24$ & layer $1$: $512$; $2$: $26$; $3$: $1$; $7$: $20$; $14$: $320$; $18$: $9$; $19$: $253$; $20$: $436$; $21$: $169$; $22$: $7$; $23$: $69$; $24$: $401$\\
DINOv2-B, no registers & $12$ & layer $1$: $2$; $11$: $153$; $12$: $2$\\
DeiT-B                 & $12$ & layer $11$: $61$; $12$: $511$\\
ViT-B/16               & $12$ & layer $7$: $142$; $8$: $434$; $9$: $452$; $10$: $433$; $11$: $234$; $12$: $511$\\
RoBERTa-base           & $12$ & layer $8$: $3$; $9$: $427$; $12$: $204$\\
DistilBERT             & $6$  & layer $6$: $98$\\
\bottomrule
\end{tabular}
\caption{Layers at which some input fails condition~\eqref{eq:cond}, with the number of
failing inputs out of $512$, on content tokens and on all tokens. All other
layers are exact for every input.}
\label{tab:perlayer}
\end{table}

\section{Robustness to the input size}\label{app:ntokens}

Each model was trained at a fixed input size. This appendix therefore does not
compare models; it checks, within each model, whether the per-input verdict of
Section~\ref{sec:perinput} changes when the model is evaluated at other input
sizes.

Figure~\ref{fig:ntokens} repeats the per-input test on content tokens for DINOv2-B
and ViT-B/16 at resolutions of $112$, $224$, $336$ and $448$ pixels ($64$--$1024$
and $49$--$784$ patches), and for the two language models at sequence lengths of
$32$, $64$, $128$, $256$ and $512$ tokens ($30$--$510$ text tokens). We use $256$
images per resolution, which keeps the largest resolutions affordable and at
$224$ pixels reproduces the $512$-image shares of Table~\ref{tab:perinput}, and
$512$ sequences per length. ViT-B/16 was trained at $224$ pixels, and the
language models on sequences of up to $512$ tokens.

The verdicts do not change. DINOv2-B is exact at $99.8$--$99.97\%$ of image--layer
pairs at every resolution, with its few failures confined to layer~$1$. ViT-B/16
stays between $61\%$ and $65\%$, failing only at layers $7$--$12$. DistilBERT
remains exact for essentially every input. RoBERTa declines slightly, from
$99.97\%$ to $98.1\%$, as its failures at layers $9$ and $12$ become more frequent
with length. The averaged verdict of every model is likewise the same at every
input size. Exactness also does not peak at the training size: RoBERTa, trained
only on full-length sequences of $512$ tokens \citep{liu2019roberta}, is exact
least often at that length.

What does change is the size of the failures. In ViT-B/16 their median falls from
$12.6\%$ to $1.2\%$ of the attention's part, $\tfrac12\kappa(\bar A)$, as the
number of patches increases, consistent with Remark~\ref{rem:size}. In RoBERTa
the median rises slightly, from $4.2\%$ to $5.4\%$: at its failing layers the
coefficient is close to $1$, so attention there is concentrated, and the bound of
Remark~\ref{rem:size} does not force the failures to be small.

\IfFileExists{fig_ntokens.pdf}{%
\begin{figure}[!ht]\centering
\includegraphics[width=\textwidth]{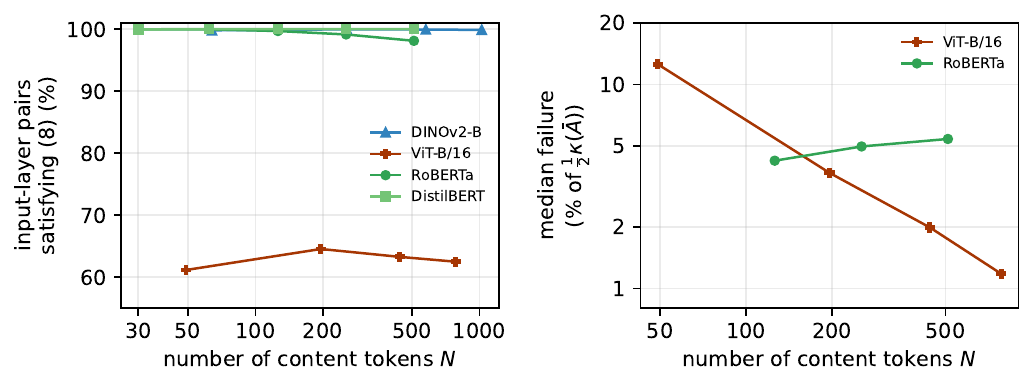}
\caption{Per-input test on content tokens at several input sizes; $N$ is the
number of content tokens. Left: percentage of input--layer pairs whose
head-averaged attention matrix satisfies condition~\eqref{eq:cond}. Right: median
gap among the failing pairs, as a percentage of $\tfrac12\kappa(\bar A)$, for the
two models with enough failures for a median; RoBERTa is shown from $N=126$.}
\label{fig:ntokens}\end{figure}}{}

\FloatBarrier
\renewcommand{\bibfont}{\small}
\bibliographystyle{plainnat}
\bibliography{reference}

\end{document}